\documentclass{ecai}
\usepackage[utf8]{inputenc}
\usepackage[T1]{fontenc}
\usepackage{times}

\usepackage{amsmath}
\usepackage{amssymb}
\usepackage{amsthm}
\usepackage{cite}

\usepackage[usenames,dvipsnames,svgnames,table]{xcolor}
\usepackage[unicode,pdfmenubar,linktoc=all,hidelinks,bookmarks, pdftex]{hyperref}
\usepackage{graphicx}
\usepackage{xspace}
\usepackage{xparse}
\usepackage{enumerate}
\usepackage{tabularx}
\usepackage{booktabs}
\usepackage{csquotes}

\usepackage{latexsym}
\usepackage{stmaryrd}
\usepackage{wasysym}
\usepackage{multicol}
\usepackage{subcaption}

\usepackage{tikz}
\usetikzlibrary{fadings}
\usetikzlibrary{shapes,decorations,positioning}
\usetikzlibrary{fadings,shapes,decorations,positioning,calc}
\usetikzlibrary{decorations.markings,decorations.pathreplacing}
\usetikzlibrary{shapes.geometric,automata,positioning}
\usetikzlibrary{shadows}\usetikzlibrary{calc}
\usetikzlibrary{decorations,decorations.pathmorphing,decorations.pathreplacing}
\usetikzlibrary{calc}

\newtheorem{theorem}{Theorem}
\newtheorem{proposition}{Proposition}

\theoremstyle{definition}
\newtheorem{definition}{Definition}
\newtheorem{example}{Example}

\DeclareFontFamily{U}{mathb}{\hyphenchar\font45}
\DeclareFontShape{U}{mathb}{m}{n}{
	<-6> mathb5 <6-7> mathb6 <7-8> mathb7
	<8-9> mathb8 <9-10> mathb9
	<10-12> mathb10 <12-> mathb12
}{}
\DeclareSymbolFont{mathb}{U}{mathb}{m}{n}
\DeclareMathSymbol{\llcurly}{\mathrel}{mathb}{"CE}
\DeclareMathSymbol{\ggcurly}{\mathrel}{mathb}{"CF}

\DeclareMathSymbol{\centerdot}{\mathrel}{mathb}{"0D}

\usepackage{xspace}
\usepackage{graphicx}
\usepackage{booktabs}
\usepackage{arydshln}
\usepackage{enumerate}
\usepackage{csquotes}
\usepackage{xparse}
\usepackage{wasysym}

\makeatletter
\newcommand{\ifLatexThree}[2]{\@ifpackageloaded{xparse}{#1}{#2}}
\newcommand{\ifAMSmath}[2]{\@ifpackageloaded{amsmath}{#1}{#2}}
\newcommand{\ifMathSCR}[2]{\@ifpackageloaded{mathrsfs}{#1}{#2}}
\newcommand{\ifMathHyperREF}[2]{\@ifpackageloaded{hyperref}{#1}{#2}}
\makeatother

\ifLatexThree{%
	\NewDocumentCommand{\headword}{s o m}{\IfBooleanTF{#1}{#3}{\textbf{#3}}\IfNoValueTF{#2}{\index{#3}}{\index{#2}}}%
}{%
	\makeatletter
	\def\@headword#1{\textbf{#1}\index{#1}}%
	\def\@@headword#1{#1\index{#1}}%
	\def\headword#1{\@ifstar\@headword{#1}\@@headword{#1}}%
	\makeatother
}

\makeatletter
\@ifundefined{naturals}{}{}
\@ifundefined{ol}{}{}
\makeatother

\ifx\nmableit\undefined
\makeatletter
\ifx\centernot\undefined
\newcommand*{\centernot}{%
	\mathpalette\@centernot
}
\fi
\def\@centernot#1#2{%
	\mathrel{%
		\rlap{%
			\settowidth\dimen@{$\m@th#1{#2}$}%
			\kern.5\dimen@
			\settowidth\dimen@{$\m@th#1=$}%
			\kern-.5\dimen@
			$\m@th#1\not$%
		}%
		{#2}%
	}%
}
\makeatother
\DeclareRobustCommand\nmableitSymb{\mathrel|\mkern-.5mu\joinrel\sim} %
\newcommand{\nmableit}{\ensuremath{\mbox{$\,\nmableitSymb\,$}}} %
\fi

\makeatletter
\ifMathHyperREF{%
	\newcommand{\seqref}[1]{\hyperref[{#1}]{\textup{\tagform@split{\getrefnumber{#1}}}}}%
}{%
	\newcommand{\seqref}[1]{\textup{\tagform@split{\getrefnumber{#1}}}}%
}
\newcommand\tagform@split[1]{%
	\begingroup
	\m@th\normalfont(\ignorespaces #1\unskip\@@italiccorr)%
	\endgroup
}
\makeatother

\ifAMSmath{%
\makeatletter
\newcommand{\leqnomode}{\tagsleft@true\let\veqno\@@leqno}
\newcommand{\reqnomode}{\tagsleft@false\let\veqno\@@eqno}
\makeatother
}{}
\newcommand{\ksIF}{\text{if }}
\newcommand{\ksTHEN}{\text{, then }}
\newcommand{\ksAND}{\text{ and }}
\newcommand{\ksOR}{\text{ or }}

\newcommand{\ksIFFlong}{\text{ if and only if }}

\newcommand{\modelsOf}[1]{\ensuremath{\llbracket #1\rrbracket}}
\newcommand{\modelsOfES}[1]{\ensuremath{\| #1 \|}}
\newcommand{\negOf}[1]{{\ensuremath{\neg{#1}}}}

\newcommand{\beliefsOf}[1]{\ensuremath{Bel\left(#1\right)}}

\newcommand{\setAllES}{\ensuremath{\mathcal{E}}}

\newcommand{\propLang}{\ensuremath{\mathcal{L}}}

\newcommand{\ramseyCond}[2]{\ensuremath{(\,#1\,|\,#2\,)}}
\newcommand{\contractionCond}[2]{\ensuremath{[\,#1\,|\,#2\,]}}

\ifMathSCR{%
}{%
}

	\makeatletter
\newcommand{\textlabel}[2]{%
	\protected@edef\@currentlabel{#1}%
	\phantomsection%
	#1\label{#2}%
}
\makeatother

\newcommand{\change}{\ensuremath{\div}}
\makeatletter
\renewcommand{\leqnomode}{\tagsleft@true}
\renewcommand{\reqnomode}{\tagsleft@false}
\makeatother
\renewcommand{\beliefsOf}[1]{\ensuremath{\text{Bel}\!\left(#1\right)}}

\ecaisubmission   %

\begin{document}
\title{A Conditional Perspective for Iterated Belief Contraction}

\author{Kai Sauerwald\institute{FernUniversität in Hagen, Germany, kai.sauerwald@fernuni-hagen.de} \and Gabriele Kern-Isberner\institute{Technical University Dortmund, Germany, gabriele.kern-isberner@cs.tu-dortmund.de}  \and Christoph Beierle\institute{FernUniversität in Hagen, Germany, christoph.beierle@fernuni-hagen.de}}

\maketitle

\begin{abstract}
	According to Boutillier, Darwiche and Pearl and others, principles for iterated revision can be characterised in terms of changing beliefs about conditionals.
	For iterated contraction a similar formulation is not known. %
This is especially because for iterated belief change the connection between revision and contraction via the Levi and Harper identity is not straightforward, and therefore, characterisation results do not transfer easily between iterated revision and contraction.
	In this article, we develop an axiomatisation of iterated contraction in terms of changing conditional beliefs. 
	We prove that the new set of postulates conforms semantically to the class of operators like the ones given by Konieczny and Pino Pérez for iterated contraction.

\end{abstract} 
\setlength{\abovedisplayskip}{5.3pt}
\setlength{\belowdisplayskip}{5.3pt}

\section{Introduction}
\label{sec:introduction}

For the three main classes of theory change, revision, expansion and contraction, different characterisations are known \cite{KS_FermeHansson2018}, which are heavily supported by the correspondence between revision and contraction via the Levi and Harper identities \cite{KS_Levi1977,KS_Harper1976}.
The situation is different for iterated belief change, focussing on belief change operators which, due to their nature, can be applied iteratively and thus, to more than one epistemic state. 
In this field, one of the most influential articles is the seminal paper \cite{KS_DarwichePearl1997} by Darwiche and Pearl (DP), establishing the insight that belief sets are not a sufficient representation for iterated belief revision.
An agent has to encode more information about her belief change strategy into her \emph{epistemic state} - where the revision strategy deeply corresponds with conditional beliefs.
This requires additional postulates that guarantee intended behaviour in forthcoming changes, especially that the possibilities of changing conditional beliefs is limited.
The common way of semantic encoding, also established by Darwiche and Pearl \cite{KS_DarwichePearl1997}, is an extension of Katsuno and Mendelzon's characterisation of the class of revisions by Alchourr{\'{o}}n, G{\"{a}}rdenfors and Makinson \cite{KS_AlchourronGaerdenforsMakinson1985} in terms of plausibility orderings \cite{KS_KatsunoMendelzon1992}, where it is assumed that the epistemic states contain an order of the worlds (or interpretations).

Similar work has been done in recent years for iterated contraction. 
Caridroit, Konieczny and Marquis \cite{KS_CaridroitKoniecznyMarquis2015} provided postulates for contraction in propositional logic and a characterisation with plausibility orders in the style of Katsuno and Mendelzon \cite{KS_KatsunoMendelzon1992}.
By this characterisation, the main characteristic of a contraction  with $ \alpha $ is  that the worlds of the previous state remain plausible and that the most plausible counter-models of $ \alpha $ become plausible.
Chopra, Ghose, Meyer and Wong \cite{KS_ChopraGhoseMeyerWong2008} transferred these to the Darwiche-Pearl framework of epistemic states, contributed semantic postulates for contraction on epistemic states in the fashion of Darwiche and Pearl, and equivalent syntactic postulates that depend on a revision function.
In the same framework, Konieczny and {Pino P{\'{e}}rez} provided additional syntactic iteration postulates for contraction which are independent from revision operators \cite{KS_KoniecznyPinoPerez2017}. %
However, none of these approaches on iterated contraction provides a focus on conditionals like the work by Darwiche and Pearl \cite{KS_DarwichePearl1997}.

In this article, we develop a new set of syntactic postulates for iterated contraction. 
These new postulates for iterated contraction are formulated in the fashion of Darwiche and Pearl.
We show that our set of postulates and the set of postulates given by Konieczny and {Pino P{\'{e}}rez} \cite{KS_KoniecznyPinoPerez2017} define the same class of contraction operators in the light of the basic postulates.
However, we argue that our new postulates highlight new aspects of iterated contraction operators.
Especially, the new postulates highlight the specific role of conditionals in the same manner as the postulates for iterative revision by Darwich and Pearl do.
For this we use specific conditionals for contraction, also called contractionals, which are studied by Bochman \cite{KS_Bochman2001}.
To develop some of the new postulates, we define an equivalence relation for epistemic states with respect to a proposition.
Furthermore, we argue that the new postulates are more succinct; dealing less with changes of disjunctive beliefs.
Succintness of postulates is of particular importance when concepts on iterated belief change developed for changes in propositional logic are translated to other formalisms \cite{KS_ShapiroPagnuccoLesperanceLevesque2011,KS_DelgrandePeppasWoltran2013,KS_DelgrandePeppas2015,KS_DelgrandePeppasWoltran2018}, and also when belief contraction is used for modelling phenomena, like forgetting (see the recent survey \cite{KS_EiterKern-Isberner2019}).

In summary the main contributions of this article are:
		\vspace{-0.2cm}
\begin{itemize}
	\item Postulates for iterated contraction and conditional beliefs 
	\item A notion of relative equivalence for epistemic states
	\item Succinct iterated contraction postulates under relative equivalence
	\item Representation theorems for the sets of postulates
\end{itemize}
\vspace{-0.2cm}

The rest of the paper is organised as follows. Section \ref{sec:prelim} provides the technical background, especially on belief change.
In Section \ref{sec:conditional_contraction}, the role of conditional beliefs is explained, and contractionals and $ \alpha $-equivalence is introduced.
In Section \ref{sec:postulates}, both new sets of postulates for iterated contraction are proposed and characterisation results for them are proven,
finally resulting in an extended representation theorem for iterated contraction. 
Section \ref{sec:conclusion} concludes and points out future work.

\section{Formal Background}
\label{sec:prelim}
We start by recalling basics of propositional logic and total preorders.
	
	\vspace{-0.2cm}
	\subsection{Propositional Logic}

Let $ \Sigma $ be a propositional signature (non empty finite set of propositional variables) and $ \propLang $ a propositional language over $ \Sigma $. 
With lower Greek letters $ \alpha,\beta,\gamma,\ldots $ we denote formulas in $ \propLang $ and with lower case letters $ a,b,c,\ldots$ propositional variables from $ \Sigma $.
	The set of 
	propositional interpretations $ \Omega $, also called set of worlds, is identified with the set of corresponding complete conjunctions over $ \Sigma $.
Propositional entailment is denoted by $ \models $, the set of models of $ \alpha $ with $ \modelsOf{\alpha} $, and $ Cn(\alpha)=\{ \beta\mid \alpha\models \beta \} $ is the deductive closure of $ \alpha $.
For a set $ X $ we define $ Cn(X)=\{ \beta \mid X\models\beta \} $. 
For a set of worlds $ \Omega'\subseteq \Omega $ and  a total preorder $ \leq $ (total, reflexive and transitive relation) over $ \Omega $, we denote with $ \min(\Omega',\leq)=\{ \omega\mid  \omega\in\Omega' \text{ and } \forall \omega'\in\Omega'\ \omega\leq \omega' \} $ the set of all worlds in the lowest layer of $ \leq $ that are elements in $ \Omega' $.
For a total preorder $ \leq $, we denote with $ < $ its strict variant, i.e. $ x < y $ iff $ x \leq y $ and $ y \not\leq x $; 
and we write $ x \simeq y $ iff $ x \leq y $ and $ y\leq x $.

\subsection{Epistemic States and Belief Changes}
AGM theory\cite{KS_AlchourronGaerdenforsMakinson1985}, by Alchourr{\'{o}}n, G{\"{a}}rdenfors and Makinson, deals with belief change in the context of belief sets, i.e., deductively closed sets of
propositions. 
In contrast, the area of iterated belief change abstracts from a belief set to an \headword{epistemic state}%
, sometimes also called belief state, in which the agent maintains all necessary information for her belief apparatus. 
With $ \setAllES $ we denote the set of all epistemic states over $ \propLang $.
Without defining what an epistemic state is, we assume that for every epistemic state $ \Psi\in\setAllES $ we can obtain the set of plausible sentences $ \beliefsOf{\Psi}\subseteq \mathcal{L} $ of $ \Psi $, which is deductively closed.
We write $ \Psi\models\alpha $ iff $ \alpha\in\beliefsOf{\Psi} $ and we define $ \modelsOfES{\Psi}=\{ \omega \mid \omega\models \alpha \text{ for each } \alpha\in\beliefsOf{\Psi} \} $.
A 
belief change operator over $\mathcal{L}$ is a function $ \circ : \setAllES \times \propLang \to 
\setAllES $.

Katsuno and Mendelzon \cite{KS_KatsunoMendelzon1992} propose that an epistemic state $ \Psi $ should be equipped with an ordering $ \leq_{\Psi} $ of the worlds (interpretations),
where the compatibility with $ \beliefsOf{\Psi} $ is ensured by the so-called faithfulness. 
\begin{definition}[Faithful Assignment \cite{KS_KatsunoMendelzon1992}]\label{def:faithful_assignment}
	A function $ \Psi\mapsto \leq_\Psi $ that maps each epistemic state to a total preorder on interpretations is said to be a faithful assignment if and only if:
	\leqnomode
\begin{align*}
& \ksIF \omega_1 \in \modelsOfES{\Psi} \ksAND \omega_2 \in \modelsOfES{\Psi} \ksTHEN \omega_1 \simeq_\Psi \omega_2  \tag{FA1}\label{pstl:FA1} \\
& \ksIF \omega_1 \in \modelsOfES{\Psi} \ksAND \omega_2\notin\modelsOfES{\Psi} \ksTHEN \omega_1 <_\Psi \omega_2 \tag{FA2}\label{pstl:FA2}
\end{align*}
\end{definition}
	Intuitively, $ \leq_{\Psi} $ orders the worlds by plausibility, such that the minimal worlds with respect to $ \leq_{\Psi} $ are the most plausible worlds.

\subsection{Iterated Revision}

	Revision deals with the problem of incorporating new beliefs into an agents belief set, thereby maintaining consistency.
	The well-known approach to revision given by AGM \cite{KS_AlchourronGaerdenforsMakinson1985}, which implements the principle of minimal change, has a counterpart in the framework of epistemic states.
\begin{proposition}[AGM Revision for Epistemic States {\cite{KS_DarwichePearl1997}}]\label{prop:es_revision}
	A belief change operator $ * $ is an \emph{AGM revision operator for epistemic states} if there is a faithful assignment $ \Psi\mapsto \leq_\Psi $ such that:
	\begin{equation}\label{eq:repr_es_revision}
	\modelsOfES{\Psi * \alpha} = \min(\modelsOf{\alpha},\leq_{\Psi})
	\end{equation}
\end{proposition}	
\noindent  Driven by the insight that iteration needs additional constraints, Darwiche and Pearl proposed the following postulates:
\begingroup\leqnomode\begin{align*}
& \ksIF \alpha\models\mu \ksTHEN \beliefsOf{\Psi * \mu *\alpha} = \beliefsOf{\Psi * \alpha} 
\tag{DP1}\label{pstl:DP1} \\
& \ksIF \alpha\models\negOf{\mu} \ksTHEN \beliefsOf{\Psi * \mu *\alpha} = \beliefsOf{\Psi * \alpha}
\tag{DP2}\label{pstl:DP2} \\
& \ksIF \Psi*\alpha \models \mu \ksTHEN (\Psi * \mu) *\alpha \models \mu
\tag{DP3}\label{pstl:DP3} \\
& \ksIF \Psi*\alpha \not\models \negOf{\mu} \ksTHEN (\Psi * \mu) *\alpha \not\models \negOf{\mu}
\tag{DP4}\label{pstl:DP4} 
\end{align*}
It is well-known that these operators can be characterised in the semantic framework of total preorders.
\begin{proposition}[Iterated Revision{\cite{KS_DarwichePearl1997}}]\label{prop:it_es_revision}
	Let $ * $ be an AGM revision operator for epistemic states. Then $ * $ satisfies \eqref{pstl:DP1} to \eqref{pstl:DP4} if and only there exists a faithful assignment $ \Psi\mapsto\leq_{\Psi} $ such that \eqref{eq:repr_es_revision} and the following postulates are satisfied:
	\begingroup
\begin{description}
	\item[\normalfont(\textlabel{RR8}{pstl:RR8})] \( \ksIF \omega_1,\omega_2 \in \modelsOf{\alpha} \ksTHEN \omega_1 \!\leq_{\Psi}\! \omega_2 \Leftrightarrow \omega_1 \!\leq_{\Psi * \alpha}\! \omega_2 \) 
	\smallskip
	\item[\normalfont(\textlabel{RR9}{pstl:RR9})] \( \ksIF \omega_1,\omega_2 \in \modelsOf{\negOf{\alpha}} \ksTHEN \omega_1 \!\leq_{\Psi}\! \omega_2 \Leftrightarrow \omega_1 \!\leq_{\Psi * \alpha}\! \omega_2 \)
	\smallskip
	\item[\normalfont(\textlabel{RR10}{pstl:RR10})] \( \ksIF \omega_1 \!\in\! \modelsOf{\alpha} \ksAND \omega_2 \!\in\! \modelsOf{\negOf{\alpha}}  \ksTHEN    \omega_1 \!<_{\Psi}\! \omega_2 \! \Rightarrow \! \omega_1 \!<_{\Psi * \alpha}\! \omega_2 \)
	\smallskip
	\item[\normalfont(\textlabel{RR11}{pstl:RR11})] \( \ksIF \omega_1  \!\in\! \modelsOf{\alpha} \ksAND \omega_2 \!\in\! \modelsOf{\negOf{\alpha}}  \ksTHEN    \omega_1 \!\leq_{\Psi}\! \omega_2 \! \Rightarrow \! \omega_1 \!\leq_{\Psi * \alpha}\! \omega_2 \)
\end{description}
	\endgroup
\end{proposition}%

\subsection{Iterated Contraction}

Contraction is the problem  of withdrawing beliefs.
Postulates for AGM contraction in the framework of epistemic states where given by Chopra, Ghose, Meyer and  Wong \cite{KS_ChopraGhoseMeyerWong2008}, and by Caridroit, Konieczny and Marquis \cite{KS_CaridroitKoniecznyMarquis2015} for propositional formula. 
Here, we give the formulation by Chropra et al. \cite{KS_ChopraGhoseMeyerWong2008}:
\begin{description}
	\item[\normalfont(\textlabel{C1}{pstl:C1})] \( \beliefsOf{\Psi \change \alpha} \subseteq \beliefsOf{\Psi}  \) 
	\smallskip
	\item[\normalfont(\textlabel{C2}{pstl:C2})] \( \ksIF  \alpha\notin\beliefsOf{\Psi} \ksTHEN \beliefsOf{\Psi}\subseteq \beliefsOf{\Psi \change\alpha} \) 
	\smallskip
	\item[\normalfont(\textlabel{C3}{pstl:C3})] \( \ksIF \alpha \not\equiv \top \ksTHEN  \alpha \notin \beliefsOf{\Psi  \change  \alpha} \) 
	\smallskip
	\item[\normalfont(\textlabel{C4}{pstl:C4})] \( \beliefsOf{\Psi} \subseteq Cn(\beliefsOf{\Psi  \change  \alpha} \cup \{\alpha\}) \) 
	\smallskip
	\item[\normalfont(\textlabel{C5}{pstl:C5})] \( \ksIF \alpha \equiv\beta \ksTHEN \beliefsOf{\Psi  \change  \alpha} = \beliefsOf{\Psi  \change  \beta} \) 
	\smallskip
	\item[\normalfont(\textlabel{C6}{pstl:C6})] \( \beliefsOf{\Psi \change \alpha} \cap \beliefsOf{\Psi \change \beta} \subseteq \beliefsOf{\Psi  \change  (\alpha\land\beta)} \) 
	\smallskip
	\item[\normalfont(\textlabel{C7}{pstl:C7})] \( \ksIF \beta\!\notin\!\beliefsOf{\!\Psi \!  \change \! (\alpha\!\land\!\beta)\!}  \ksTHEN \beliefsOf{\Psi \! \change \! (\alpha\!\land\!\beta)} \subseteq \beliefsOf{\Psi\! \change \!\beta} \) 
\end{description}
For an explanation of these postulates we refer to the article by Caridroit et al. \cite{KS_CaridroitKoniecznyMarquis2015}.
A characterisation in terms of total preorders on epistemic states is given by the following proposition.
\begin{proposition}[AGM Contraction for Epistemic States {\cite{KS_KoniecznyPinoPerez2017}}]\label{prop:es_contraction}
A belief change operator $ \change $ fulfils the postulates \eqref{pstl:C1} to \eqref{pstl:C7} if and only if there is a faithful assignment $ \Psi\mapsto \leq_\Psi $ such that:
\begin{equation}\label{eq:repr_es_contraction}
\modelsOfES{\Psi \change \alpha} = \modelsOfES{\Psi} \cup \min(\modelsOf{\negOf{\alpha}},\leq_{\Psi})
\end{equation}
\end{proposition}
\noindent For the purpose of the article we will say a belief change operator $ \change $ is an \emph{AGM contraction operator for epistemic states}, if $ \change $ fulfils \eqref{pstl:C1} to \eqref{pstl:C7}.
The postulates \eqref{pstl:C1} to \eqref{pstl:C7} do not explicitly state how one should maintain the contraction strategy in the case of iteration. 
It is desirable to support AGM contraction operators for epistemic states by additional postulates for iteration. 
Konieczny and Pino P{\'{e}}rez give postulates\footnote{The original formulation \cite{KS_KoniecznyPinoPerez2017} uses a formula $ B(\Psi) $ instead of a belief set $ \beliefsOf{\Psi} $.} for intended iteration behaviour of contraction \cite{KS_KoniecznyPinoPerez2017}:
\begin{description}
	\item[\normalfont(\textlabel{IC8}{pstl:KPP8})] \( \ksIF \!\negOf{\alpha} \!\models\! \gamma \text{, then} 
	\left(\begin{aligned} &
	\beliefsOf{\!\Psi \!\change\! \alpha} \!\subseteq\!   \beliefsOf{\!\Psi \!\change\! (\alpha\!\lor\!\beta)} \\
\!\Leftrightarrow\ 
	&%
	\beliefsOf{\!\Psi \!\change\! \gamma \!\change\! \alpha} \!\subseteq\!   \beliefsOf{\!\Psi \!\change\! \gamma \change (\alpha\!\lor\!\beta)}
	\end{aligned}
	\right) \) 
	\smallskip
	\item[\normalfont(\textlabel{IC9}{pstl:KPP9})] \( \ksIF \gamma  \!\models\! \alpha \text{, then}
	\left(\begin{aligned}
	& \beliefsOf{\!\Psi \!\change\! \alpha} \!\subseteq\!  \beliefsOf{\!\Psi \!\change\! (\alpha\lor\beta)} \\
\!\Leftrightarrow\ 
	&%
	\beliefsOf{\!\Psi \!\change\! \gamma \!\change\! \alpha} \!\!\subseteq\!  \beliefsOf{\!\Psi \!\change\! \gamma \!\change\! (\alpha\lor\beta)}
	\end{aligned}\right) \) 
	\smallskip
	\item[\normalfont(\textlabel{IC10}{pstl:KPP10})] \( \ksIF \negOf{\beta}\!\models\! \gamma \text{, then}\! 
	\left(
	\begin{aligned}
	& \beliefsOf{\Psi \!\change\! \gamma \!\change\! \alpha} \!\subseteq\!  \beliefsOf{\Psi \!\change\! \gamma \!\change\! (\alpha\lor\beta)} \\
	&   \! \Rightarrow \beliefsOf{\Psi \!\change\! \alpha} \!\subseteq\! \beliefsOf{\Psi \!\change\! (\alpha\lor\beta)}
	\end{aligned}
	\right) \) 
	\smallskip
	\item[\normalfont(\textlabel{IC11}{pstl:KPP11})] \( \ksIF \gamma \!\models\! \beta \text{, then} 
	\left(
	\begin{aligned}
	& \beliefsOf{\Psi \!\change\! \gamma \!\change\! \alpha} \!\subseteq\!  \beliefsOf{\Psi \!\change\! \gamma \!\change\! (\alpha\lor\beta)} \\
	&\! \Rightarrow \beliefsOf{\Psi \!\change\! \alpha} \!\subseteq\!  \beliefsOf{\Psi \!\change\! (\alpha\lor\beta)}
	\end{aligned}
	\right) \) 
\end{description}

\noindent For an explanation of \eqref{pstl:KPP8} to \eqref{pstl:KPP11} we refer to Konieczny and Pino P{\'{e}}rez \cite{KS_KoniecznyPinoPerez2017}. The class of operators fulfilling these postulates is captured semantically by the following representation theorem.
\begin{proposition}[Iterated Contraction{\cite{KS_KoniecznyPinoPerez2017}}]\label{prop:it_es_contraction}
Let $ \change $ be an AGM contraction operator for epistemic states. Then $ \change $ satisfies \eqref{pstl:KPP8} to \eqref{pstl:KPP11} if and only there exists a faithful assignment $ \Psi\mapsto\leq_{\Psi} $ such that \eqref{eq:repr_es_contraction} and the following postulates are satisfied:
\vspace{-\medskipamount}
\begin{description}
	\item[\normalfont(\textlabel{CR8}{pstl:CR8})] \( \ksIF \omega_1,\omega_2 \in \modelsOf{\alpha} \ksTHEN \omega_1 \leq_{\Psi} \omega_2 \Leftrightarrow \omega_1 \leq_{\Psi\change\alpha} \omega_2 \)
	\smallskip
	\item[\normalfont(\textlabel{CR9}{pstl:CR9})] \( \ksIF \omega_1,\omega_2 \in \modelsOf{\negOf{\alpha}} \ksTHEN \omega_1 \leq_{\Psi} \omega_2 \Leftrightarrow \omega_1 \leq_{\Psi\change\alpha} \omega_2 \)
	\smallskip
	\item[\normalfont(\textlabel{CR10}{pstl:CR10})] \( \ksIF \omega_1\!\in\!\modelsOf{\negOf{\alpha}} \ksAND \omega_2\!\in\!\modelsOf{\alpha}  \ksTHEN    \omega_1 \!<_{\Psi}\! \omega_2 \!\Rightarrow\! \omega_1 \!<_{\Psi\change\alpha}\! \omega_2 \)
	\smallskip
	\item[\normalfont(\textlabel{CR11}{pstl:CR11})] \( \ksIF \omega_1\!\in\!\modelsOf{\negOf{\alpha}} \ksAND \omega_2\!\in\!\modelsOf{\alpha}  \ksTHEN    \omega_1 \!\leq_{\Psi}\! \omega_2 \!\Rightarrow\! \omega_1 \!\leq_{\Psi\change\alpha}\! \omega_2 \)
\end{description}
\end{proposition}
In the semantic perspective, the postulates \eqref{pstl:CR8} and \eqref{pstl:CR11} ensure that the order of worlds does not
change if they are equivalent in the perspective of contraction. 
The postulates and \eqref{pstl:CR10} and \eqref{pstl:CR11} enforce that no world that contradicts the contracted information
is getting relatively more plausible after the contraction than a world that was already more plausible
before and does not contradict the contracted information \cite{KS_Kern-IsbernerBockSauerwaldBeierle2017a}.

\section{Conditionals and Belief Change}
\label{sec:conditional_contraction}
In the following, we will give some background on the interrelation
	between conditionals and revisions. 
	Then we introduce conditionals
	which are related to contractions, called contractionals. 
	We also introduce $ \alpha $-equivalence, which is shown to be related to the acceptance of contractionals in an epistemic state.

\subsection{Ramsey Test and Iterated Revision}

One of the insights of belief-change theory is that a conditional belief \enquote{if $ \alpha $, then usually $ \beta $}, denoted here by $ \ramseyCond{\beta}{\alpha} $, is related to a revision by the so-called Ramsey test \cite{KS_Stalnaker1968}:
\begingroup%
\begin{description}
	\item[\normalfont(\textlabel{RT}{ramseytest})] \( \Psi \models \ramseyCond{\beta}{\alpha} \ksIFFlong \Psi * \alpha \models \beta \) 
\end{description}
\endgroup
Note that we did not yet define what $ \Psi \models \ramseyCond{\beta}{\alpha} $ means. 
For this, we take the right side of \eqref{ramseytest} as definition, i.e., $ \ramseyCond{\beta}{\alpha} $ is accepted in $ \Psi $, denoted by $ \Psi \models \ramseyCond{\beta}{\alpha} $, if $ \Psi * \alpha \models \beta $. 
The conditional $ \ramseyCond{\beta}{\alpha} $ therefore depends then on $ * $, and thus a notation like $ \ramseyCond{\beta}{\alpha}^* $ would be more correct, but we omit the superscript here, since the context is always clear.

The implication for iterated belief change is that the maintenance of the change strategy of an agent is related to the change of conditional beliefs.
This leads to the representation by total preorders, since conditionals can be related to total preorders. 
Let $ \leq $ be a total preorder over $ \Omega $.
We say a Ramsey conditional $ \ramseyCond{\beta}{\alpha} $ is accepted in $ \leq $ if for every $ \omega_\text{f} \in \modelsOf{\alpha\land\negOf{\beta}}  $ there exists an $ \omega_\text{v} \in \modelsOf{\alpha\land\beta} $ such that $ \omega_\text{v} < \omega_\text{f} $. 
A formal statement about the correspondence between acceptance and Ramsey test is given by the following proposition.
\begin{proposition}[\cite{KS_DarwichePearl1997}]
	Let $ * $ be an AGM revision operator for epistemic states and $ \Psi\mapsto\leq_{\Psi} $ be a corresponding faithful assignment.
	Then $  \ramseyCond{\beta}{\alpha} $ is accepted in $ \leq_{\Psi} $ if and only if $ \Psi * \alpha \models \beta $.
\end{proposition}

Darwiche and Pearl \cite{KS_DarwichePearl1997} propose the following interpretation of their iteration postulates \eqref{pstl:DP1} to \eqref{pstl:DP4} in terms of conditionals:
\begin{description}
	\item[\normalfont(\textlabel{DP1$_\text{cond}$}{pstl:DP1cond})] \( \ksIF \alpha\models\mu \ksTHEN \Psi \models \ramseyCond{\beta}{\alpha} \Leftrightarrow \Psi * \mu \models \ramseyCond{\beta}{\alpha} \) 
	\smallskip
	\item[\normalfont(\textlabel{DP2$_\text{cond}$}{pstl:DP2cond})] \( \ksIF \alpha\models\negOf{\mu} \ksTHEN \Psi \models \ramseyCond{\beta}{\alpha} \Leftrightarrow \Psi * \mu \models \ramseyCond{\beta}{\alpha} \) 
	\smallskip
	\item[\normalfont(\textlabel{DP3$_\text{cond}$}{pstl:DP3cond})] \( \ksIF \Psi \models \ramseyCond{\mu}{\alpha} \ksTHEN \Psi * \mu \models \ramseyCond{\mu}{\alpha}  \) 
	\smallskip
	\item[\normalfont(\textlabel{DP4$_\text{cond}$}{pstl:DP4cond})] \( \ksIF \Psi \not\models \ramseyCond{\negOf{\mu}}{\alpha} \ksTHEN \Psi * \mu \not\models \ramseyCond{\negOf{\mu}}{\alpha}  \) 
\end{description}
The formulation of \eqref{pstl:DP1cond} to \eqref{pstl:DP4cond} highlights that the iterative belief revision postulates \eqref{pstl:DP1} to \eqref{pstl:DP4} enforce minimal change to conditional beliefs under certain conditions.

\subsection{Contractionals}
By the Ramsey test, conditionals are connected to a revision operator. In an analogue way, one might think of conditionals which are connected to a contraction operator.
	To distinguish such conditionals from Ramsey test conditionals, we will call them \emph{contractionals} \cite{KS_Bochman2001} and denote them by $ \contractionCond{\beta}{\alpha} $.
We suggest to read a \emph{contractional} $ \contractionCond{\beta}{\alpha} $ as \textit{\enquote{belief $ \beta $ even in the absence of $ \alpha $}}. 
More formally, we define that $ \Psi $ \emph{accepts} $ \contractionCond{\beta}{\alpha} $, written $ \Psi \models \contractionCond{\beta}{\alpha} $, if $ \Psi \change \alpha \models \beta  $. Thus, analogue to \eqref{ramseytest}, we have the following correspondence:
	\begingroup%
	\begin{description}
		\item[\normalfont(\textlabel{Contractional}{contractionConditional})] \(  \Psi \models \contractionCond{\beta}{\alpha} \ksIFFlong \Psi \change \alpha \models \beta \) 
	\end{description}
	\endgroup
Consider the following example for a comparison of the meaning of contractionals to Ramsey test conditionals.
\begin{example}
	Let $ f $ have the intended meaning that something is \enquote{able to fly} and $ p $ the intended meaning that something is a \enquote{penguin}.
	Then the acceptance of a (Ramsey test) conditional $ \ramseyCond{\neg f}{p} $ states that if the agent is getting aware that something is a penguin, she will believe that it is not able to fly.
	In contrast, the acceptance of a contractional $ \contractionCond{\neg f}{p} $ states that the agent keeps the belief that something is not able to fly, even if the agent 
	gives up her belief that it is a penguin.
\end{example}

Like conditionals, contractionals can be related to total preorders.
We say that the contractional $ \contractionCond{\beta}{\alpha} $ is accepted in a total preorder $ \leq $ over $ \Omega $ if $ \min(\Omega,\leq) \subseteq \modelsOf{\beta} $ and for every $ \omega_1 \in \modelsOf{\negOf{\alpha}\land\negOf{\beta}}  $ there exists an $ \omega_2 \in \modelsOf{\negOf{\alpha}\land\beta} $ such that $ \omega_2 < \omega_1 $. 
Supposing that $ \change $ is an AGM contraction for epistemic states, then, via Proposition \ref{prop:es_contraction}, a contractional $ \contractionCond{\beta}{\alpha} $ is accepted in $ \leq_{\Psi} $ if and only if it is accepted in $ \Psi $.

\begin{figure*}[t]

\begin{tabular}{ll|ll}
	\toprule
	                                                                                                     \multicolumn{2}{c}{Postulates using contractionals}                                                                                                       &                                               \multicolumn{2}{c}{Equivalent non-conditional postulates}                                                \\ \midrule
	(C8$_\text{cond}$)  & \( \ksIF \negOf{\alpha}\models\beta    \ksTHEN  \Psi\!\change\!\alpha \!\models\! \contractionCond{\gamma\!\lor\! \negOf{\alpha} }{\beta}   \!\Leftrightarrow\! \Psi \!\models\! \contractionCond{\gamma\lor \negOf{\alpha}}{\beta}   \) & (C8)  & \( \ksIF \negOf{\alpha}\models\beta  \ksTHEN \beliefsOf{\Psi\change\alpha\change\beta} =_\alpha \beliefsOf{\Psi\change\beta}  \)               \\
	(C9$_\text{cond}$)  & \( \ksIF \alpha\models\beta             \ksTHEN   \Psi\!\change\!\alpha \models \contractionCond{\gamma\lor{\beta}}{\beta}  \Leftrightarrow \Psi \models \contractionCond{\gamma\lor{\beta}}{\beta} \)                                   & (C9)  & \( \ksIF \alpha\models\beta	\ksTHEN \beliefsOf{\Psi\change\alpha\change\beta} =_\negOf{\beta} \beliefsOf{\Psi\change\beta} \)                  \\
	(C10$_\text{cond}$) & \( \ksIF \negOf{\alpha}\models\gamma \ksTHEN \Psi \models \contractionCond{\gamma}{\beta}  \Rightarrow\  \Psi\change\alpha  \models \contractionCond{\gamma}{\beta} \)                                                                   & (C10) & \( \ksIF \negOf{\alpha}\models\gamma \ksTHEN  \Psi\change\beta \models \gamma \ \Rightarrow\  \Psi\change\alpha \change\beta \models \gamma \) \\
	(C11$_\text{cond}$) & \( \ksIF \alpha\models\gamma \ksTHEN  \Psi\change\alpha  \models \contractionCond{\gamma}{\beta} \Rightarrow \Psi \models \contractionCond{\gamma}{\beta} \)                                                                             & (C11) & \( \ksIF \alpha\models\gamma \ksTHEN \Psi \change\alpha \change\beta \models \gamma \Rightarrow \Psi\change\beta \models \gamma  \)            \\ \bottomrule
\end{tabular}

\caption{Overview of the two sets of postulates for iterated contraction developed in this article.}
\label{fig:postulate_overview} \end{figure*}
\subsection{$ \alpha $-Equivalence and Belief Contraction}
\label{sec:relative_change}
	For the case of iterated contraction, we need to restrain the notion of equivalence of formulas to specific cases.
In particular, we propose a notion of equivalence which is relative to a proposition $ \alpha $, which we call $ \alpha $-equivalence.
\begin{definition}[$ \alpha $-equivalence]
	For two sets of interpretations $ \Omega_1,\Omega_2\subseteq\Omega $ and a formula $ \alpha $ we say $ \Omega_1 $ is \emph{$ \alpha $-equivalent} to $ \Omega_2 $, written $ \Omega_1 =_\alpha \Omega_2 $, if $ \Omega_1 $ and $ \Omega_2 $ contain the same set of models of $ \alpha $, 
	i.e. $ \Omega_1 \cap \modelsOf{\alpha} = \Omega_2 \cap \modelsOf{\alpha} $.
\end{definition}
\noindent This is lifted to two sets of formulas $ X,Y $, by saying $ X $ is \emph{$ \alpha $-equivalent} to $ Y $, written $ X =_\alpha Y $, if $ \modelsOf{X}=_\alpha \modelsOf{Y} $.
For sets of formulas an alternative formulation is possible:
\begin{proposition}
	Two sets of formulas $ X$ and $Y $ are $ \alpha $-equivalent if and only if $ Cn(X\cup \{\alpha\} ) = Cn(Y\cup \{\alpha\} ) $.
\end{proposition}

\noindent  Intuitively, $ X $ and $ Y $ are $ \alpha $-equivalent if they agree on everything about $ \alpha $.  
In the following we give an example which demonstrates $ \alpha $-equivalence.
\begin{example}
	Suppose a scenario about birds ($ b $), penguins ($ p $) and flying ($ f $).
	Let $ X=Cn(b \land f, p \to f) $ and $ Y=Cn(b \land f, p \to \neg f) $ be belief sets which differ mainly in their beliefs about whether a penguin can fly or not.
	The models of these two belief sets are	$ \modelsOf{X}=\{ bfp, bf\overline{p} \} $ and $ \modelsOf{Y}=\{ bf\overline{p} \} $.
	Then $ X $ and $ Y $ agree in their view on birds that are no penguins, $ X =_{b\land \neg p} Y $, but they do not agree in everything about birds, $ X \neq_{b} Y $.
\end{example}

We will use the notion of $ \alpha $-equivalence as a tool to describe invariants for belief changes. 
As an example consider the following proposition, holding for every AGM contraction.
\begin{proposition}
	For every AGM contraction operator for epistemic states $ \change $ and all propositions $ \alpha,\beta $ the following postulate holds: 
	\begin{equation*}
	\ksIF \negOf{\alpha}\land\beta \equiv \bot \ksTHEN \beliefsOf{\Psi} =_\beta \beliefsOf{\Psi \change \alpha}
	\end{equation*}
\end{proposition}
\begin{proof}
		Assume $ \alpha,\beta $ such that $ \neg\alpha\land\beta\equiv\bot $. 
		By Proposition \ref{prop:es_contraction} there is a faithful assignment $ \Psi\mapsto\leq_{\Psi} $ such that \eqref{eq:repr_es_contraction} is fulfilled.
		Since $ \negOf{\alpha}\land\beta \equiv \bot $ holds, $ \negOf{\alpha} $ and $ \beta $ have no models in common.
		Therefore, we can infer that the set $ \min(\modelsOf{\negOf{\alpha}},\leq_{\Psi}) $ contains no models of $ \beta $. 
		Thus from $ \modelsOfES{\Psi\change\alpha}=\modelsOfES{\Psi} \cup \min(\modelsOf{\negOf{\alpha}},\leq_{\Psi}) $ we can derive $ \modelsOfES{\Psi\change\alpha} =_\beta  \modelsOfES{\Psi} $, which is equivalent to $ \beliefsOf{\Psi} =_\beta \beliefsOf{\Psi \change \alpha} $.
\end{proof}

	The following proposition relates $ \alpha $-equivalence of beliefs to the acceptance of contractionals.	
	\begin{proposition}\label{prop:eq_contractional_correspondence}
		Let $ \change $ be an AGM contraction operator for epistemic states, $ \Psi,\Phi $ be epistemic states and $ \alpha,\beta $ propositions. 
		Then $ \beliefsOf{\Psi\change\beta} =_\alpha \beliefsOf{\Phi\div\beta} $ holds if and only if for all propositions $ \gamma $  we have $ \Psi\models\contractionCond{\alpha \rightarrow \gamma}{\beta} \Leftrightarrow \Phi\models\contractionCond{\alpha \rightarrow \gamma}{\beta} $.
	\end{proposition}
	\begin{proof}

For the \enquote{only if} direction let $ \beliefsOf{\Psi\change\beta} =_\alpha \beliefsOf{\Phi\change\beta} $. This is equivalent to:
\begin{equation}
\modelsOfES{\Psi\change\beta}\cap \modelsOf{\alpha} = \modelsOf{\Phi\change\beta} \cap \modelsOf{\alpha} \label{eq:p11:aequivalence}
\end{equation}
Assume now (without loss of generality) that $ \Psi \models \contractionCond{\gamma \lor \negOf{\alpha}}{\beta} $ and $  \Phi \not\models \contractionCond{\gamma \lor \negOf{\alpha}}{\beta} $.
Then we get an contradiction, since there must be a world $ \omega\in\modelsOf{ \negOf{\gamma}\land \alpha } $ such that $ \omega\in\modelsOf{\Phi\change\beta} $, which contradicts the assumption in combination with Equation \eqref{eq:p11:aequivalence}.

In the \enquote{if} direction, for all propositions $ \gamma $  it holds hat $ \Psi\models\contractionCond{\gamma \lor \negOf{\alpha}}{\beta} \Leftrightarrow \Phi\models\contractionCond{\gamma \lor \negOf{\alpha}}{\beta} $.
Towards a contradiction assume now that  $ \modelsOfES{\Psi\change\beta}\cap \modelsOf{\alpha} \neq \modelsOf{\Phi\change\beta} \cap \modelsOf{\alpha} $.
This implies (without loss of generality)  that there is a world $ \omega $ such that $ \omega\notin \modelsOfES{\Psi\change\beta}\cap \modelsOf{\alpha} $ but $ \omega \in \modelsOfES{\Phi\change\beta}\cap \modelsOf{\alpha} $.
Now let $\gamma$ be a formula such that $ \modelsOf{\gamma}=\modelsOfES{\Psi\change\beta}\cap \modelsOf{\alpha} $. 
Clearly, it holds that $ \Psi \change \beta \models \gamma\lor\negOf{\alpha} $ and $ \Phi \change \beta \not\models \gamma\lor\negOf{\alpha} $.
By the correspondence between contractionals and contractions this is a contradiction the our assumption. 	\end{proof}

\section{Postulates for Iterated Contraction}
\label{sec:postulates}

While for the Darwiche-Pearl iteration postulates for revision a translation to postulates about changing conditionals via the Ramsey test is easy, the postulates \eqref{pstl:KPP8} to \eqref{pstl:KPP11} do not allow for an easy translation into postulates about belief change for contractionals.
	
In the rest of the article we develop an equivalent contractional representation for iterated contraction.
		We will also give a conditional variant of these new postulates. The two groups of postulates are summarised in Figure \ref{fig:postulate_overview}.
		The postulates have been developed from \eqref{pstl:CR8} to \eqref{pstl:CR11}, which are equivalent to \eqref{pstl:KPP8} to \eqref{pstl:KPP11} due to Proposition \ref{prop:it_es_contraction}. 
\subsection{Syntactic Postulates for \eqref{pstl:CR8} and \eqref{pstl:CR9}}
\label{sec:relative_change_itr}

In the following, for iterated contraction we define two principles which correspond to \eqref{pstl:CR8} and \eqref{pstl:CR9}, specifying situations in which beliefs after a contraction are not influenced by specific prior contractions.
Both \eqref{pstl:CR8} and \eqref{pstl:CR9} state that worlds that can not be distinguished from the point of view of a the contracted formula, should not change their plausibility. 
In particular, \eqref{pstl:CR8} enforces this condition for the models of  $ \alpha $, thus it is natural to specify the following postulate:
\begin{description}
	\item[\normalfont(\textlabel{C8}{pstl:C8})]\( \ksIF \negOf{\alpha}\models\beta  \ksTHEN \beliefsOf{\Psi\change\alpha\change\beta} =_\alpha \beliefsOf{\Psi\change\beta}  \) \\
	\emph{Explanation:} The beliefs about $ \alpha $ after a contraction with $ \beta $ are independent from whether $ \alpha $ was contracted previously or not, if $ \beta $ is more general than the negation of $ \alpha $.
\end{description}
The postulate \eqref{pstl:CR9} enforces the same condition for models of $ \negOf{\alpha} $,
which is captured by the following postulate:
\begin{description}
	\item[\normalfont(\textlabel{C9}{pstl:C9})]\( \ksIF \alpha\models\beta	\ksTHEN \beliefsOf{\Psi\change\alpha\change\beta} =_\negOf{\beta} \beliefsOf{\Psi\change\beta} \) \\
	\emph{Explanation:} The beliefs about $ \neg\beta $ after a contraction with $ \beta $ are independent from whether $ \alpha $ was contracted previously or not, if $ \beta $ is more general than $ \alpha $.
\end{description}

\noindent
We will now show that \eqref{pstl:C8} and \eqref{pstl:C9} are equivalent to \eqref{pstl:CR8} and \eqref{pstl:CR9} for AGM contraction operators on epistemic states.
\begin{proposition}\label{prop:eqrelated_c8_c9}
	Let $ \change $  be an AGM contraction operator for epistemic states. Then the following statements are equivalent:
	\begin{enumerate}[(a)]
		\item The operator $ \change $ satisfies the postulates \eqref{pstl:C8} and \eqref{pstl:C9}.
		\item There is a faithful assignment $ \Psi\mapsto\leq_{\Psi} $ related to $ \change $ by \eqref{eq:repr_es_contraction} such that \eqref{pstl:CR8} and \eqref{pstl:CR9} are satisfied.
	\end{enumerate}
\end{proposition}
\begin{proof}
\begingroup%
We will show that (a) implies (b) and that (b) implies (a).

\textbf{Part I:}
We start with the (b) to (a) direction. Let $ \change $ be an AGM contraction operator for epistemic states, thus fulfilling \eqref{pstl:C1} to \eqref{pstl:C7}, and let $ \Psi\mapsto\leq_{\Psi} $ be a faithful assignment related to $ \change $ by \eqref{eq:repr_es_contraction} such that \eqref{pstl:CR8} to \eqref{pstl:CR9} are fulfilled.
We show that \eqref{pstl:C8} and \eqref{pstl:C9} are satisfied:
\begin{description}
	\item[\eqref{pstl:C8}] 
	Let $ \negOf{\beta}\models\alpha $, which is equivalent to $ \negOf{\alpha}\models\beta $. 
	As in the proof of \eqref{pstl:CR9}, Equation \eqref{eq:repr_es_contraction} implies
	\begin{align}
	\modelsOfES{\Psi\change\alpha\change\beta}  & = \modelsOfES{\Psi} \cup \min(\modelsOf{\negOf{\alpha}},\leq_{\Psi}) \cup \min(\modelsOf{\negOf{\beta}},\leq_{\Psi\change\alpha}), \label{eq:proof:C8:4}
	\end{align}
	for every $ \alpha,\beta $.
	Furthermore, by \eqref{pstl:CR8} and $ \negOf{\beta}\models\alpha $ it holds that:
	\begin{equation}
      \min(\modelsOf{\negOf{\beta}},\leq_{\Psi}) = \min(\modelsOf{\negOf{\beta}},\leq_{\Psi\change\alpha}) \label{eq:proof:C8:5}
	\end{equation}
	Combining \eqref{eq:proof:C8:4} with \eqref{eq:proof:C8:5} yields:
	\begin{equation}
	\modelsOfES{\Psi\change\alpha\change\beta}  = \modelsOfES{\Psi} \cup \min(\modelsOf{\negOf{\alpha}},\leq_{\Psi}) \cup \min(\modelsOf{\negOf{\beta}},\leq_{\Psi}) \label{eq:proof:C8:6}
	\end{equation}
	By using $ \modelsOfES{\Psi\change\beta} = \modelsOfES{\Psi} \cup \min(\modelsOf{\negOf{\beta}},\leq_{\Psi}) $ obtained from Equation \eqref{eq:repr_es_contraction} and using $ \min(\modelsOf{\negOf{\alpha}},\leq_{\Psi})  \cap \modelsOf{\alpha} = \emptyset $,  from \eqref{eq:proof:C8:6} we conclude that
	\begin{equation}
	\modelsOfES{\Psi\change\alpha\change\beta} \cap \modelsOf{\alpha} = \modelsOfES{\Psi\change\beta} \cap \modelsOf{\alpha}
	\end{equation}
	holds,
	which is equivalent to the required result $ \beliefsOf{\Psi\change\alpha\change\beta} =_\alpha \beliefsOf{\Psi\change\beta} $.
	\item[\eqref{pstl:C9}] %
	Let $ \negOf{\beta}\models\negOf{\alpha} $, which is equivalent to $ \alpha\models\beta $. 
		By \eqref{pstl:CR9} and $ \negOf{\beta}\models\negOf{\alpha} $ it holds that:
		\begin{equation}
		\min(\modelsOf{\negOf{\beta}},\leq_{\Psi}) = \min(\modelsOf{\negOf{\beta}},\leq_{\Psi\change\alpha}) \label{eq:proof:C9:5}
		\end{equation}
		Combining Equation 
		\eqref{eq:proof:C8:4} that holds for every $ \alpha,\beta $
		with \eqref{eq:proof:C9:5} yields:
		\begin{equation}
		\modelsOfES{\Psi\change\alpha\change\beta}  = \modelsOfES{\Psi} \cup \min(\modelsOf{\negOf{\alpha}},\leq_{\Psi}) \cup \min(\modelsOf{\negOf{\beta}},\leq_{\Psi}) \label{eq:proof:C9:6}
		\end{equation}
		By using $ \negOf{\beta}\models\negOf{\alpha} $ we conclude that there are only two possible cases:
		\begin{align}
		\min(\modelsOf{\negOf{\alpha}},\leq_{\Psi})  \cap \modelsOf{\negOf{\beta}} = \emptyset ,\ksOR \label{eq:proof:C9:7} \\
		\min(\modelsOf{\negOf{\alpha}},\leq_{\Psi})  \cap \modelsOf{\negOf{\beta}} = \min(\modelsOf{\negOf{\beta}},\leq_{\Psi})  \cap \modelsOf{\negOf{\beta}}  . \label{eq:proof:C9:8} 
		\end{align}
		In both cases, \eqref{eq:proof:C9:7} and \eqref{eq:proof:C9:8}, from \eqref{eq:proof:C9:6} we directly infer:
		\begin{equation}
		\modelsOfES{\Psi\change\alpha\change\beta}  \cap \modelsOf{\negOf{\beta}}  = \modelsOfES{\Psi} \cup \min(\modelsOf{\negOf{\beta}},\leq_{\Psi})  \cap \modelsOf{\negOf{\beta}}  \label{eq:proof:C9:9}
		\end{equation}
		From \eqref{eq:proof:C9:9} and 
		$ \modelsOfES{\Psi\change\beta} = \modelsOfES{\Psi} \cup \min(\modelsOf{\negOf{\beta}},\leq_{\Psi}) $, obtained from Equation \eqref{eq:repr_es_contraction},
		we conclude that
		\begin{equation}
		\modelsOfES{\Psi\change\alpha\change\beta} \cap \modelsOf{\negOf{\beta}} = \modelsOfES{\Psi\change\beta} \cap \modelsOf{\negOf{\beta}}
		\end{equation}
		holds,
		which is equivalent to $ \beliefsOf{\Psi\change\alpha\change\beta} =_\negOf{\beta} \beliefsOf{\Psi\change\beta} $.
\end{description}

\textbf{Part II:}
For the (a) to (b) direction suppose that $ \change $ is an AGM contraction operator for epistemic states. Further assume that $ \div $ satisfies \eqref{pstl:C8} and \eqref{pstl:C9}.
By Proposition \ref{prop:es_contraction} there exists a faithful assignment $ \Psi\mapsto \leq_\Psi $ such that for every proposition $ \alpha $ Equation \eqref{eq:repr_es_contraction} holds.
We will show that $ \Psi\mapsto \leq_\Psi $ satisfies \eqref{pstl:CR8} and \eqref{pstl:CR9}.
\begin{description}
	\item[\eqref{pstl:CR8}]  Suppose $ \omega_1,\omega_2\in\modelsOf{\alpha} $. We choose $ \beta=\negOf{(\omega_1\lor\omega_2)} $ and therefore, we have $ \negOf{\alpha}\models\beta $ and $ \negOf{\beta}\models\alpha $.
	By \eqref{pstl:C8} we have $ \beliefsOf{\Psi\change\alpha\change\beta}=_\alpha\beliefsOf{\Psi\change\beta} $, which implies:
	\begin{equation}
	\modelsOfES{\Psi\change\beta} =_\alpha \modelsOfES{\Psi\change\alpha\change\beta} \label{eq:ual:CR8},
	\end{equation}
	which is equivalent to $ \modelsOfES{\Psi\change\beta} \cap \modelsOf{\alpha} = \modelsOfES{\Psi\change\alpha\change\beta} \cap \modelsOf{\alpha} $.
	From Equation \eqref{eq:repr_es_contraction} we obtain that
	\begin{align}
	\modelsOfES{\Psi\change\alpha\change\beta} & =\modelsOfES{\Psi\change\alpha}\cup\min(\modelsOf{\negOf{\beta}},\leq_{\Psi\change\alpha})                                  \notag \\
	& =\modelsOfES{\Psi}\cup\min(\modelsOf{\negOf{\alpha}},\leq_{\Psi})\cup\min(\modelsOf{\negOf{\beta}},\leq_{\Psi\change\alpha}) \label{eq:CR8:2}
	\end{align} 
	and
	\begin{equation}
	\modelsOfES{\Psi\change\beta}=\modelsOfES{\Psi}\cup\min(\modelsOf{\negOf{\beta}},\leq_{\Psi}) . \label{eq:CR8:3}
	\end{equation}
	Substituting \eqref{eq:CR8:2} and \eqref{eq:CR8:3}  into Equation \eqref{eq:ual:CR8} leads to 
	\begin{multline}
	\modelsOfES{\Psi}\cup\min(\modelsOf{\negOf{\alpha}},\leq_{\Psi})\cup\min(\modelsOf{\negOf{\beta}},\leq_{\Psi\change\alpha}) \\ =_\alpha 
	\modelsOfES{\Psi}\cup\min(\modelsOf{\negOf{\beta}},\leq_{\Psi}) . \label{eq:CRX:1}
	\end{multline}
	Equation \eqref{eq:CRX:1} is equivalent to:
	\begin{multline}
	\left( \modelsOfES{\Psi}\cup\min(\modelsOf{\negOf{\alpha}},\leq_{\Psi})\cup\min(\modelsOf{\negOf{\beta}},\leq_{\Psi\change\alpha}) \right) \cap \modelsOf{\alpha} \\
	 = 	\left( \modelsOfES{\Psi}\cup\min(\modelsOf{\negOf{\beta}},\leq_{\Psi}) \right) \cap \modelsOf{\alpha} \label{eq:CRX:2}
	\end{multline}
	Because $ \min(\modelsOf{\negOf{\alpha}},\leq_{\Psi}) \!\cap\! \modelsOf{\alpha}\!=\!\emptyset $, Equation \eqref{eq:CRX:1} is equivalent to:
	\begin{multline}
	\left( \modelsOfES{\Psi}\cup\min(\modelsOf{\negOf{\beta}},\leq_{\Psi\change\alpha}) \right) \cap \modelsOf{\alpha} \\ =
	\left( \modelsOfES{\Psi}\cup\min(\modelsOf{\negOf{\beta}},\leq_{\Psi}) \right) \cap \modelsOf{\alpha} \label{eq:CRX:3}
	\end{multline}
	Remember that $ \negOf{\beta}\models\alpha $ and therefore $ \modelsOf{\negOf{\beta}}\subseteq\modelsOf{\alpha} $.
	Equation \eqref{eq:CRX:3} implies
	\begin{multline}
	\left( \modelsOfES{\Psi}\cup\min(\modelsOf{\negOf{\beta}},\leq_{\Psi\change\alpha}) \right) \cap \modelsOf{\negOf{\beta}} \\
	 =
	\left( \modelsOfES{\Psi}\cup\min(\modelsOf{\negOf{\beta}},\leq_{\Psi}) \right) \cap \modelsOf{\negOf{\beta}} \label{eq:CRX:4},
	\end{multline}%
	which is equivalent to $ \modelsOfES{\Psi}\cup\min(\modelsOf{\negOf{\beta}},\leq_{\Psi\change\alpha}) =_\negOf{\beta} \modelsOfES{\Psi}\cup\min(\modelsOf{\negOf{\beta}},\leq_{\Psi}) $.

	We will now show that the following holds:
	\begin{equation}
	\min(\modelsOf{\negOf{\beta}},\leq_{\Psi\change\alpha}) = \min(\modelsOf{\negOf{\beta}},\leq_{\Psi}) \label{eq:CR8:4}
	\end{equation}
	\textbf{Case 1:} Consider the case of $ \modelsOfES{\Psi}\cap\modelsOf{\negOf{\beta}}\neq\emptyset $. 
	Because of the faithfulness of the assignment, it must be the case that $  \min(\modelsOf{\negOf{\beta}},\leq_{\Psi}) \subseteq \modelsOfES{\Psi} $. 
	Moreover, $ \min(\modelsOf{\negOf{\beta}},\leq_{\Psi}) $ is exactly the set of models of $ \negOf{\beta} $ contained in $ \modelsOfES{\Psi} $, i.e.:
	\begin{equation}
	\modelsOfES{\Psi} \cap \modelsOf{\negOf{\beta}} = \min(\modelsOf{\negOf{\beta}},\leq_{\Psi}) \label{eq:CRX:5}
	\end{equation}
	From Equation \eqref{eq:repr_es_contraction} in Proposition \eqref{prop:es_contraction}
	we easily get
	\begin{equation}
	\modelsOfES{\Psi\change\alpha} \cap \modelsOf{\alpha} = \left( \modelsOfES{\Psi}\cup\min(\modelsOf{\negOf{\alpha}},\leq_{\Psi}) \right) \cap \modelsOf{\alpha},
	\end{equation}
	which is, because of $ \min(\models(\negOf{\alpha}),\leq_{\Psi}) \cap \modelsOf{\alpha}=\emptyset $, equivalent to:
	\begin{equation}
	\modelsOfES{\Psi\change\alpha} \cap \modelsOf{\alpha} = \modelsOfES{\Psi} \cap \modelsOf{\alpha}
	\end{equation}
	Since $ \negOf{\beta}\models\alpha $, 		
	the set $ \modelsOfES{\Psi\change\alpha} $ contains the same models of $ \negOf{\beta} $ as $ \modelsOfES{\Psi} $, i.e.:
	\begin{equation}
	\modelsOfES{\Psi\change\alpha} \cap \modelsOf{\negOf{\beta}} = \modelsOfES{\Psi} \cap \modelsOf{\negOf{\beta}} \label{eq:CRX:6}
	\end{equation}
	Due to our assumption $ \modelsOfES{\Psi}\cap\modelsOf{\negOf{\beta}} \neq \emptyset $, this implies $ \modelsOfES{\Psi\change\alpha}\cap\modelsOf{\negOf{\beta}}\neq\emptyset $. 
	Then, because of the faithfulness of the assignment, it must be the case that $  \min(\modelsOf{\negOf{\beta}},\leq_{\Psi\change\alpha}) \subseteq \modelsOfES{\Psi\change\alpha} $. 
	Moreover, $ \min(\modelsOf{\negOf{\beta}},\leq_{\Psi\change\alpha}) $ is exactly the set of models of $ \negOf{\beta} $ contained in $ \modelsOfES{\Psi\change\alpha} $, i.e.:
	\begin{equation}
	\modelsOfES{\Psi\change\alpha} \cap \modelsOf{\negOf{\beta}} = \min(\modelsOf{\negOf{\beta}},\leq_{\Psi\change\alpha}) \label{eq:CRX:7}
	\end{equation}
	Together, Equations \eqref{eq:CRX:5}, \eqref{eq:CRX:6} and \eqref{eq:CRX:7}, imply
	Equation \eqref{eq:CR8:4} in
	this case.

	\textbf{Case 2:} We now consider the other case $ \modelsOfES{\Psi}\cap\modelsOf{\negOf{\beta}}=\emptyset $.
	Since $ \negOf{\beta}\models\alpha $ and $ \min(\modelsOf{\negOf{\alpha}},\leq_{\Psi}) $ contains no models of $ \alpha $, it must be the case that
	\begin{equation}
	\modelsOfES{\Psi}\cup\min(\modelsOf{\negOf{\beta}},\leq_{\Psi\change\alpha}) =_\alpha 
	\modelsOfES{\Psi}\cup\min(\modelsOf{\negOf{\beta}},\leq_{\Psi}) .  \label{eq:CR8:3a}
	\end{equation}
	We directly conclude from Equation \eqref{eq:CR8:3a} that Equation \eqref{eq:CR8:4} holds, which finishes the proof of Equation \eqref{eq:CR8:4}.

	Note that $ \modelsOf{\negOf{\beta}} $ has only two elements, $ \modelsOf{\negOf{\beta}}=\{\omega_1,\omega_2\} \subseteq \modelsOf{\alpha} $, and thus information about the minima provides us the relative order of the two elements $ \omega_1$ and $\omega_2 $. So, from Equation \eqref{eq:CR8:4}, we can conclude that $ \omega_1 \leq_\Psi \omega_2 $ if and only if $ \omega_1 \leq_{\Psi\change\alpha} \omega_2 $.
	\medskip
	\item[\eqref{pstl:CR9}] Suppose $ \omega_1,\omega_2\in\modelsOf{\negOf{\alpha}} $. We choose $ \beta=\negOf{(\omega_1\lor\omega_2)} $ and therefore, we have $ \alpha\models\beta $.
	By \eqref{pstl:C9} we have $ \beliefsOf{\Psi\change\alpha\change\beta}=_\negOf{\beta}\beliefsOf{\Psi\change\beta} $, which implies:
	\begin{equation}
	\modelsOfES{\Psi\change\beta} \cap \modelsOf{\negOf{\beta}} = \modelsOfES{\Psi\change\alpha\change\beta} \cap \modelsOf{\negOf{\beta}} \label{eq:ual:CR9:1}
	\end{equation}
	From Equation \eqref{eq:repr_es_contraction} we obtain that 
	\begin{align}
	\modelsOfES{\Psi\change\alpha} & =\modelsOfES{\Psi}\cup\min(\modelsOf{\negOf{\alpha}},\leq_{\Psi})  \label{eq:CR9:X1} , \\
	\modelsOfES{\Psi\change\beta} & =\modelsOfES{\Psi}\cup\min(\modelsOf{\negOf{\beta}},\leq_{\Psi})  \label{eq:CR9:3}
	\end{align}
	and employing Equation \eqref{eq:repr_es_contraction} twice yields:		
	\begin{align}
	\modelsOfES{\Psi\change\alpha\change\beta} 
	=\modelsOfES{\Psi}\cup\min(\modelsOf{\negOf{\alpha}},\leq_{\Psi})\cup\min(\modelsOf{\negOf{\beta}},\leq_{\Psi\change\alpha}) \label{eq:CR9:2}.
	\end{align} 
	Substituting \eqref{eq:CR9:2} and \eqref{eq:CR9:3}  into Equation \eqref{eq:ual:CR9:1} leads to 
\begin{multline}
	\left( \modelsOfES{\Psi}\!\cup\!\min(\modelsOf{\negOf{\beta}},\leq_{\Psi}) \right) \!\cap\! \modelsOf{\negOf{\beta}} \\
	= \left(\modelsOfES{\Psi}\!\cup\!\min(\modelsOf{\negOf{\alpha}},\leq_{\Psi})\!\cup\!\min(\modelsOf{\negOf{\beta}},\leq_{\Psi\change\alpha}) \right) \! \cap \! \modelsOf{\negOf{\beta}}. \label{eq:ual:CR9:4}
	\end{multline}
	Note that every model of $ \negOf{\beta} $ is a model of $ \negOf{\alpha} $, therefore, either one of the following holds:
	\begin{align}		\min(\modelsOf{\negOf{\alpha}},\leq_{\Psi}) \cap \modelsOf{\negOf{\beta}} & = \emptyset ,\ksOR \label{eq:ual:CR9:5} \\
	\min(\modelsOf{\negOf{\alpha}},\leq_{\Psi}) \cap \modelsOf{\negOf{\beta}} & = \min(\modelsOf{\negOf{\beta}},\leq_{\Psi})  . \label{eq:ual:CR9:6} 
	\end{align}
	In the case of \eqref{eq:ual:CR9:5}, Equation \eqref{eq:ual:CR9:4} reduces to $  
	\left({\modelsOfES{\Psi}\cup\min(\modelsOf{\negOf{\beta}},\leq_{\Psi})}\right) \cap \modelsOf{\negOf{\beta}} 
	=
	\left({\modelsOfES{\Psi}\cup\min(\modelsOf{\negOf{\beta}},\leq_{\Psi\change\alpha})}\right) \cap \modelsOf{\negOf{\beta}} $. 
	Furthermore, from \eqref{eq:ual:CR9:5} and $ \negOf{\beta}\models\negOf{\alpha} $ we conclude $ \modelsOfES{\Psi}\cap\modelsOf{\negOf{\beta}}=\emptyset $. This allows to conclude $ \min(\modelsOf{\negOf{\beta}},\leq_{\Psi\change\alpha}) = {\min(\modelsOf{\negOf{\beta}},\leq_{\Psi})} $.
	
	For the other case, the case of \eqref{eq:ual:CR9:6}, note that $ \min(\modelsOf{\negOf{\alpha}},\leq_{\Psi}) \subseteq \modelsOfES{\Psi \change \alpha} $. 
	By Equation \eqref{eq:ual:CR9:6} it must hold that $ {\min(\modelsOf{\negOf{\beta}},\leq_{\Psi})} \subseteq {\modelsOfES{\Psi \change \alpha}}  $. By the faithfulness of $ \Psi\mapsto\leq_{\Psi} $ (in particular condition \eqref{pstl:FA1} in Definition \ref{def:faithful_assignment}) we have $ \min(\modelsOf{\negOf{\beta}},\leq_{\Psi})=\min(\modelsOf{\negOf{\beta}},\leq_{\Psi\change\alpha}) $, because the minimal models of $ \negOf{\beta} $ with respect to $ \leq_{\Psi\change\alpha} $ are contained in $ \modelsOfES{\Psi \change \alpha} $ by \eqref{eq:CR9:X1} and \eqref{eq:ual:CR9:6}.
	
	In summary, in both cases, \eqref{eq:ual:CR9:5} and \eqref{eq:ual:CR9:6}, we can conclude:
	\begin{equation}
	\min(\modelsOf{\negOf{\beta}},\leq_{\Psi\change\alpha}) = \min(\modelsOf{\negOf{\beta}},\leq_{\Psi}) \label{eq:ual:CR9:7}
	\end{equation}
	Note that $ \modelsOf{\negOf{\beta}} $ has only two elements, $ \modelsOf{\negOf{\beta}}=\{\omega_1,\omega_2\} \subseteq \modelsOf{\negOf{\alpha}} $, and thus information about the minima provides us the relative order of the two elements $ \omega_1$ and $\omega_2 $. From Equation \eqref{eq:ual:CR9:7} we can conclude that $ \omega_1 \leq_\Psi \omega_2 $ if and only if $ \omega_1 \leq_{\Psi\change\alpha} \omega_2 $.\qedhere
\end{description}
\endgroup \end{proof}

One might wonder, why we not simply use the following postulate as syntactic counterpart to 
\eqref{pstl:CR9}:
	\begin{description}
	\item[\normalfont(\textlabel{C9$ ^\prime $}{pstl:C9dash})]\( \ksIF \alpha\models\beta	\ksTHEN \beliefsOf{\Psi\change\alpha\change\beta} =_\negOf{\alpha} \beliefsOf{\Psi\change\beta} \) 
\end{description}
The following proposition shows that this would not hold.
\begin{proposition}
	There is an AGM contraction operator for epistemic states that satisfies \eqref{pstl:CR9} but violates \eqref{pstl:C9dash}.
\end{proposition}
\begin{proof}
	Consider the case where $ \alpha=a $ and $ \beta=a\lor b $. Clearly, we have $ \alpha\models\beta $ in this case.
	Table \ref{tbl:exmpl_cr9_contraction} specifies a faithful preorder for the epistemic state $ \Psi $ and for states after contraction with $ \alpha $ and $ \beta $.
	Note that none of the changes in Table \ref{tbl:exmpl_cr9_contraction} violates \eqref{pstl:CR9}. 
	But the most plausible models in $ \leq_{\Psi\change a \change (a\lor b)} $ and $ \leq_{\Psi\change (a\lor b)} $ contain different models of $ \negOf{\alpha} $. This implies that $ \beliefsOf{\Psi\change\alpha\change\beta} \neq_\negOf{\alpha} \beliefsOf{\Psi\change\beta} $, which is a violation of \eqref{pstl:C9dash}.
\end{proof}
\begin{table}
	\begin{center}
	\begin{tabular}{c|c||c|c||c}
		\toprule
		   State    &                          $ \Psi $                          &               $ \Psi \change a $                &      $ \Psi \change a \change (a\lor b) $       &    $ \Psi \change (a\lor b) $    \\ \midrule
		    TPO     &                      $ \leq_{\Psi} $                       &            $ \leq_{\Psi \change a} $            &   $ \leq_{\Psi\change a \change (a\lor b)} $    & $ \leq_{\Psi\change (a\lor b)} $ \\ \midrule\midrule
		implausible & $ a\overline{b}\ \overline{a}b\ \overline{a}\overline{b} $ &                $ a\overline{b} $                &                $ a\overline{b} $                & $ a\overline{b}\ \overline{a}b $ \\
		 plausible  &                           $ ab $                           & $ ab\ \overline{a}b\ \overline{a}\overline{b} $ & $ ab\ \overline{a}b\ \overline{a}\overline{b} $ & $ ab\ \overline{a}\overline{b} $ \\ \bottomrule
	\end{tabular} 		\caption{Example for the incompatibility between \eqref{pstl:CR9} and \eqref{pstl:C9dash}.}
		\label{tbl:exmpl_cr9_contraction}
	\end{center}
\end{table}

	The correspondence between contractionals and contractions from Section \ref{sec:conditional_contraction} and Proposition \ref{prop:eq_contractional_correspondence} allows us to give a conditional formulation of the postulates \eqref{pstl:C8} and \eqref{pstl:C9}:
\begin{description}
	\item[\normalfont(\textlabel{C8$_\text{cond}$}{pstl:C8cond})]\( \ksIF \negOf{\alpha}\models\beta    \ksTHEN  \Psi\!\change\!\alpha \!\models\! \contractionCond{\gamma\!\lor\! \negOf{\alpha} }{\beta}   \!\Leftrightarrow\! \Psi \!\models\! \contractionCond{\gamma\lor \negOf{\alpha}}{\beta}   \)\\
	\item[\normalfont(\textlabel{C9$_\text{cond}$}{pstl:C9cond})]\( \ksIF \alpha\models\beta             \ksTHEN   \Psi\!\change\!\alpha \models \contractionCond{\gamma\lor{\beta}}{\beta}  \Leftrightarrow \Psi \models \contractionCond{\gamma\lor{\beta}}{\beta} \) \\				
\end{description}

\noindent  We close with a formal statement about the interrelationship between the conditional and non-conditional variant for these postulates.
\vspace{-0.2cm}
\begin{proposition}\label{cor:c8_c9_conditional}
	Let $ \change $ be an AGM contraction operator for epistemic states. Then, \eqref{pstl:C8}, respectively \eqref{pstl:C9}, is satisfied by $ \change $ if and only if $ \eqref{pstl:C8cond} $, respectively \eqref{pstl:C9cond}, is satisfied.
\end{proposition} 
\subsection{Syntactic Postulates for \eqref{pstl:CR10} and \eqref{pstl:CR11}}
\label{sec:itr_contractionals_1011}
The postulates \eqref{pstl:CR10} and \eqref{pstl:CR11}
both ensure that by a contraction with $ \alpha $,
models of $ \alpha $ should not be improved with respect to models of $ \negOf{\alpha} $.
In the context of AGM contractions for epistemic states we use here, this is expressed by the following postulates:
	\begin{description}
		\item[\normalfont(\textlabel{C10$_\text{cond}$}{pstl:C10cond})]\( \ksIF \negOf{\alpha}\models\gamma \ksTHEN \Psi \models \contractionCond{\gamma}{\beta}  %
		\text{  implies  }
		\Psi\change\alpha  \models \contractionCond{\gamma}{\beta} \)\\[0.2em]
		\emph{Explanation:} A contraction with $ \alpha $ preserves the acceptance of a contractional if its conclusion $ \gamma $ is more general than %
		$ \negOf{\alpha} $.
	\item[\normalfont(\textlabel{C11$_\text{cond}$}{pstl:C11cond})]\( \ksIF \alpha\models\gamma \ksTHEN  \Psi\change\alpha  \models \contractionCond{\gamma}{\beta} %
	\text{ implies }
	\Psi \models \contractionCond{\gamma}{\beta} \) \\[0.2em]
	\emph{Explanation:} If a contractional whose conclusion $ \gamma $ is more general than $ \alpha $ is accepted after a contraction with $ \alpha $, then the contractional should be accepted previously.
\end{description}

	By using contraposition and the correspondence between contractionals and contractions, the following non-conditional formulation of the principles \eqref{pstl:C10cond} and \eqref{pstl:C11cond} can be obtained:
		\begin{description}
		\item[\normalfont(\textlabel{C10}{pstl:C10})]\( \ksIF \negOf{\alpha}\models\gamma \ksTHEN  \Psi\change\beta \models \gamma \ \Rightarrow\  \Psi\change\alpha \change\beta \models \gamma \) 	\\

		\vspace{-0.3cm}
		\item[\normalfont(\textlabel{C11}{pstl:C11})]\( \ksIF 
		\alpha\models\gamma
		\ksTHEN \Psi \change\alpha \change\beta \models \gamma \Rightarrow \Psi\change\beta \models \gamma  \) \\
	\end{description}

\vspace{-0.3cm}

Note that AGM contractions for epistemic states fulfil the inclusion postulate \eqref{pstl:C1}, and therefore no contraction can add additional beliefs. 
The postulates \eqref{pstl:C10} and \eqref{pstl:C11} constrain further which beliefs should be retained. The postulate \eqref{pstl:C10} ensures that a contraction with $ \alpha $ does not internally give up beliefs. 
The postulate \eqref{pstl:C11} is more difficult, stating that if two contractions do not withdraw a belief $ \gamma $, then the second contraction only does not withdraw $ \gamma $.

The following proposition states the connection between \eqref{pstl:C10cond} and \eqref{pstl:C10}, and between \eqref{pstl:C11cond} and \eqref{pstl:C11}. 
\begin{proposition}\label{cor:c10_c11_conditional}
	Let $ \change $ be an AGM contraction operator for epistemic states. Then \eqref{pstl:C10cond}, respectively \eqref{pstl:C11cond}, is satisfied $ \change $ if and only if \eqref{pstl:C10cond}, respectively \eqref{pstl:C11cond}, is fulfilled.
\end{proposition}

	We show for the non-conditional postulates \eqref{pstl:C10} and \eqref{pstl:C10} that they are related to \eqref{pstl:CR10} and \eqref{pstl:CR11}.
	\begin{proposition}\label{prop:eqrelated_c10_c11}
		Let $ \change $  be a belief change operator, satisfying \eqref{pstl:C1} to \eqref{pstl:C9}. Then the following statements are equivalent:
		\begin{enumerate}[(a)]
			\item The operator $ \change $ satisfies the postulates \eqref{pstl:C10} and \eqref{pstl:C11}
			\item There is a faithful assignment $ \Psi\mapsto\leq_{\Psi} $ related to $ \change $ by \eqref{eq:repr_es_contraction} such that \eqref{pstl:CR10} and \eqref{pstl:CR11} are satisfied.
		\end{enumerate}
	\end{proposition}
\begin{proof}
\begingroup%
In the following, we will show that (b) implies (a) and that (a) implies (b).

\textbf{Part I:} We show the (b) to (a) direction.
Let $ \change $ be a belief change operator fulfilling \eqref{pstl:C1} to \eqref{pstl:C9}. By Proposition \ref{prop:eqrelated_c8_c9}, there is a faithful assignment $ \Psi\mapsto\leq_{\Psi} $ related to $ \change $ by \eqref{eq:repr_es_contraction} such that  \eqref{pstl:CR8} and \eqref{pstl:CR9} are fulfilled.
We will show that \eqref{pstl:C10} and \eqref{pstl:C11} are satisfied.
\begin{description}
	\item[\eqref{pstl:C10}] 
	Let $ \negOf{\alpha}\models \gamma $ and $ \Psi\change\beta\models\gamma $. 
	We want to show that $ \Psi\change\alpha\change\beta\models\gamma $ holds.
	By Equation \eqref{eq:repr_es_contraction} we obtain
	\begin{equation}
	\modelsOfES{\Psi\change\beta} = \modelsOfES{\Psi} \cup \min( \modelsOf{\negOf{\beta}} , \leq_{\Psi} ) \subseteq \modelsOf{\gamma}
	\label{eq:proof:C11:1}
	\end{equation}
	and with Equation \eqref{eq:proof:C8:4} which holds for every $ \alpha,\beta $ we get:
	\begin{align}
	\modelsOfES{\Psi\!\change\!\alpha\!\change\!\beta} 
	& \!=\!  \modelsOfES{\Psi} \!\cup\! \min( \modelsOf{\negOf{\alpha}} , \leq_{\Psi} ) \!\cup\! \min( \modelsOf{\negOf{\beta}} , \leq_{\Psi\change\alpha} ) 
	.
	\label{eq:proof:C11:2}
	\end{align}
	We show that every $ \omega \in \modelsOfES{\Psi\change\alpha\change\beta} $ is a model of $ \gamma $. 
	By Equation \eqref{eq:proof:C11:2} either $ \omega\in\modelsOfES{\Psi} $, $ \omega\in\min( \modelsOf{\negOf{\alpha}} , \leq_{\Psi} ) $ or $ \omega\in\min( \modelsOf{\negOf{\beta}} , \leq_{\Psi\change\alpha} ) $. 
	For these three cases we have:
	\begin{itemize}
		\item If $ \omega\in\modelsOfES{\Psi} $, then by Equation \eqref{eq:proof:C11:1} we have $ \omega\models\gamma $.
		\item If $ \omega \in \min( \modelsOf{\negOf{\alpha}} , \leq_{\Psi\change\alpha} ) $, then by the assumption $ \negOf{\alpha}\models\gamma $ we have $ \omega\models\gamma $.
		\item For $ \omega \in \min( \modelsOf{\negOf{\beta}} , \leq_{\Psi\change\alpha} ) $ assume that $ \omega\models\negOf{\gamma} $. If $ \omega\models\negOf{\alpha} $, then $ \omega\models\gamma $ by the assumption $ \negOf{\alpha}\models \gamma $.
		Therefore, we can safely assume $ \omega\models\alpha $.
		Since $ \min( \modelsOf{\negOf{\beta}} , \leq_{\Psi} ) \subseteq \modelsOf{\gamma} $, there must be $ \omega_1\in \min( \modelsOf{\negOf{\beta}} , \leq_{\Psi} ) $ such that $ \omega_1 <_\Psi \omega $.
		If $ \omega_1,\omega \in \modelsOf{\alpha} $, then $ \omega_1 <_{\Psi\change\alpha} \omega $ by \eqref{pstl:CR8}.
		For $ \omega_1 \in \modelsOf{\negOf{\alpha}} $ and $ \omega \in \modelsOf{{\alpha}} $ we conclude $ \omega_1 <_{\Psi\change\alpha} \omega $ by \eqref{pstl:CR10}.
		Thus, it must be the case that $ \omega_1 <_{\Psi\change\alpha} \omega $, which is a contradiction to the minimality of $ \omega $ with respect to $ \leq_{\Psi\change\alpha} $.
	\end{itemize}
	Equation \eqref{eq:proof:C11:2} implies that $ \omega \models \gamma $, and therefore $ \Psi\change\alpha\change\beta\models\gamma $.
	\medskip
	\item[\eqref{pstl:C11}] Let $ \alpha,\beta,\gamma $ be such that 
	$ \alpha\models\gamma $
	and $ \Psi\change\alpha\change\beta\models\gamma $. 
	We want to show $ \Psi\change\beta\models\gamma $.
	By Equation \eqref{eq:repr_es_contraction} we have $ \modelsOfES{\Psi}\subseteq\modelsOf{\gamma} $ and $ {\min(\modelsOf{\negOf{\beta}},\leq_{\Psi\change\alpha})\subseteq\modelsOf{\gamma}} $.
	Now let $ \omega_1\in\modelsOf{\negOf{\beta}} $ such that $ \omega_1 \notin \min(\modelsOf{\negOf{\beta}},\leq_{\Psi\change\alpha}) $. We show that $ \omega_1\notin \min(\modelsOf{\negOf{\beta}},\leq_{\Psi}) $ or $ \omega_1\models\gamma $. Let $ \omega_2\in \min(\modelsOf{\negOf{\beta}},\leq_{\Psi\change\alpha}) $ and thus, $ \omega_2 <_{\Psi\change\alpha} \omega_1 $. We differentiate by cases:
	\begin{enumerate}
		\item For $ \omega_1\in\modelsOf{\alpha} $  the assumption $ \alpha\models\gamma $ immediately yields $ \omega\models\gamma $.
		\item In the case of $ \omega_1\in\modelsOf{\negOf{\alpha}} $ and $ \omega_2\in\modelsOf{\alpha} $ we conclude $ \omega_2 <_{\Psi} \omega_1 $ by contraposition of \eqref{pstl:CR11}.
		\item In the case of $ \omega_1,\omega_2\in\modelsOf{\negOf{\alpha}} $ we conclude by \eqref{pstl:CR9} that $ \omega_2 <_{\Psi} \omega_1 $ holds.
	\end{enumerate}
	This shows that either $ \omega_2 <_\Psi \omega_1 $ or $ \omega_1\models\gamma $.
	The first case implies $ \omega_1\notin \min( \modelsOf{\negOf{\beta}} ,\leq_{\Psi} ) $, and thus yields $ {\min(\modelsOf{\negOf{\beta}},\leq_{\Psi})} \subseteq \modelsOf{\gamma} $.
	In summary, we have $ \modelsOfES{\Psi\change\beta} = \modelsOfES{\Psi} \cup {\min(\modelsOf{\negOf{\beta}},\leq_{\Psi})} \subseteq \modelsOf{\gamma} $.
\end{description}

\textbf{Part II:} We show the (a) to (b) direction.
Suppose that $ \change $ is a belief change operator that satisfies \eqref{pstl:C1} to \eqref{pstl:C11}.
By Proposition \ref{prop:es_contraction} there exists a faithful assignment $ \Psi\mapsto \leq_\Psi $ such that for every proposition $ \alpha $ Equation \eqref{eq:repr_es_contraction} holds.
We will show that $ \Psi\mapsto \leq_\Psi $ satisfies \eqref{pstl:CR10} and \eqref{pstl:CR11}.
\begin{description}
	\item[\eqref{pstl:CR10}] 
	Suppose $ \omega_1\in\modelsOf{\negOf{\alpha}} $, $ \omega_2\in\modelsOf{{\alpha}} $ and $ \omega_1 <_{\Psi}\omega_2 $. We want to show $ \omega_1 <_{\Psi\change\alpha} \omega_2 $.
	For this purpose let $ \beta=\negOf{(\omega_1\lor\omega_2)} $. 
	Since $ \Psi\mapsto\leq_{\Psi} $ is a faithful assignment (especially by \eqref{pstl:FA2}) it must be the case that $ \omega_2\notin \modelsOfES{\Psi} $.
	By use of Equation \eqref{eq:repr_es_contraction} we can conclude that $ \omega_2\notin\modelsOfES{\Psi\change\beta} $ and $ \omega_1\in\modelsOfES{\Psi\change\beta} $.
	Now let $ \gamma=\gamma'\lor\negOf{\alpha} $, where $ \gamma' $ is a formula such that $ \modelsOfES{\Psi\change\beta}\cup\{ \omega_1 \}= \modelsOf{\gamma'} $.
		Thus, $ \negOf{\alpha} \models \gamma $, whereby $ \omega_1\models\gamma $ and $ \omega_2\not\models\gamma $. From \eqref{pstl:C10} we conclude $ \Psi\change\alpha\change\beta\models \gamma $.
		This implies that $ \omega_2\notin \modelsOfES{\Psi\change\alpha\change\beta} $. Note that $ \modelsOf{\negOf{\beta}}=\{\omega_1,\omega_2\} $ and thus, by Equation \eqref{eq:repr_es_contraction} it must be the case that $ \omega_1\in \modelsOfES{\Psi\change\alpha\change\beta} $ or $ \omega_2\in \modelsOfES{\Psi\change\alpha\change\beta} $.
		Since the latter leads to a contradiction, we conclude $ \omega_1\in \modelsOfES{\Psi\change\alpha\change\beta} $, and in summary $ \omega_1 <_{\Psi\change\alpha} \omega_2 $.
	\item[\eqref{pstl:CR11}] We show \eqref{pstl:CR11} by contraposition. 
	Suppose $ \omega_1\in\modelsOf{\negOf{\alpha}} $, $ \omega_2\in\modelsOf{{\alpha}} $ and $ \omega_2 <_{\Psi\change\alpha}\omega_1 $. 
	We will show $ \omega_2 <_{\Psi}\omega_1 $.
	Since $ \Psi\mapsto\leq_{\Psi} $ is a faithful assignment, and thus fulfils especially \eqref{pstl:FA2},
	it must be the case that $ \omega_1\notin \modelsOfES{\Psi\change\alpha} $.
	For $ \beta=\negOf{(\omega_1\lor\omega_2)} $ we can conclude
	by Equation \eqref{eq:repr_es_contraction} that $ \modelsOfES{\Psi\change\alpha\change\beta}=\modelsOfES{\Psi\change\alpha} \cup \min( \modelsOf{\negOf{\beta}} , \leq_{\Psi\change\alpha} ) $.
	Because of $ \omega_2 <_{\Psi\change\alpha} \omega_1 $ we have $ \modelsOfES{\Psi\change\alpha\change\beta} = \modelsOfES{\Psi\change\alpha}\cup\{ \omega_2 \} $.
	Now let $ \gamma=\gamma'\lor\alpha $, where $ \gamma' $ is a formula such that $ \modelsOf{\gamma'}=\modelsOfES{\Psi\change\alpha} \cup \{ \omega_2 \} $.
	By definition of $ \gamma $ it is the case that $ \Psi\change\alpha\change\beta\models \gamma $. 
	Furthermore, by definition $ \alpha\models\gamma $ and $ \omega_1\not\models\gamma$ and $ \omega_2\models\gamma $.
	By using \eqref{pstl:C11} we can conclude $ \Psi\change\beta \models \gamma $.
Note that by Equation \eqref{eq:repr_es_contraction} it must be the case that $ \omega_1\in \modelsOfES{\Psi\change\beta} $ or $ \omega_2\in \modelsOfES{\Psi\change\beta} $.
Since $ \Psi\change\beta \models \gamma $ the former is not possible, so by \eqref{pstl:FA2} we have $ \omega_2 <_\Psi \omega_1 $. \qedhere
\end{description}
\endgroup \end{proof}
 
\subsection{Extended Representation Theorem for Iterative Contraction}
\label{sec:extended_representation}

We now employ the results from Section \ref{sec:relative_change_itr} and Section \ref{sec:itr_contractionals_1011} to show our main theorem for our new sets of postulates, which are summarised in Figure \ref{fig:postulate_overview}.

\begin{theorem}[Extended Representation Theorem]\label{thm:contraction_ext_representation}
	Let $ \change $  be an AGM contraction operator for epistemic states. Then the following statements are equivalent:
\begin{enumerate}[(a)]
	\item The operator $ \change $ fulfils \eqref{pstl:C8} to \eqref{pstl:C11}.
	\item The operator $ \change $ fulfils \eqref{pstl:C8cond} to \eqref{pstl:C11cond}.
	\item The operator $ \change $ fulfils \eqref{pstl:KPP8} to \eqref{pstl:KPP11}.
	\item There is a faithful assignment $ \Psi\mapsto\leq_{\Psi} $ related to $ \change $ by \eqref{eq:repr_es_contraction} such that \eqref{pstl:CR8} to \eqref{pstl:CR11} are satisfied.%
\end{enumerate}
\end{theorem}
\begin{proof}
The equivalence of (c) and (d) is given by Proposition \ref{prop:it_es_contraction}. By Proposition \ref{cor:c10_c11_conditional} and Proposition \ref{cor:c8_c9_conditional} we get the equivalence between (a) and (b).
Finally, by Proposition \ref{prop:eqrelated_c8_c9}  and Proposition \ref{prop:eqrelated_c10_c11} the statements (a) and (d) are equivalent.
\end{proof}

Note that we could extend Theorem \ref{thm:contraction_ext_representation} to cover also the syntactic postulates for iterated contractions by Chopra, Ghose, Meyer and Wong \cite{KS_ChopraGhoseMeyerWong2008}.
Their contraction postulates depend on a revision; this complicates specifying a class of operators, since for instance in the iterative case there are more revisions than contractions \cite{KS_KoniecznyPinoPerez2017}.

By Theorem \ref{thm:contraction_ext_representation} the class of AGM contraction operators which fulfil the iteration postulates 
\eqref{pstl:KPP8} to \eqref{pstl:KPP11}
for contraction can be expressed equivalently by any of the two groups of postulates \eqref{pstl:C8} to \eqref{pstl:C11} and \eqref{pstl:C8cond} to \eqref{pstl:C11cond} developed here in this paper and summarised in Figure \ref{fig:postulate_overview}.
 
\section{Conclusion and Future Work}
\label{sec:conclusion}
We take a conditional perspective on iterated contraction by developing a set of postulates \eqref{pstl:C8cond} to \eqref{pstl:C11cond}, which highlight conditional beliefs that are retained by an iterative contraction.
Additionally, we provide a set of syntactic postulates \eqref{pstl:C8} to \eqref{pstl:C11} for iterated contraction which are more succinct than \eqref{pstl:KPP8} to \eqref{pstl:KPP11} proposed in \cite{KS_KoniecznyPinoPerez2017}.
Moreover, without \eqref{pstl:C8} to \eqref{pstl:C11} it would be difficult to obtain \eqref{pstl:C8cond} to \eqref{pstl:C11cond}.
We proved an extended representation theorem for iterated contraction, which shows that all these different sets of postulates describe the same set of operators.

The notion of $ \alpha $-equivalence was introduced as a form of equivalence between epistemic states with respect to a proposition $ \alpha $. 
We showed the usefulness of this relation for postulation and how it allows to provide insight about invariants in belief change.

For the postulates \eqref{pstl:C8cond} to \eqref{pstl:C11cond} we use specific conditionals, called contractionals, which are connected to contractions in the same manner as the Ramsey test draws a connection to revisions.
Contractionals have been studied in the context of inference by Bochman \cite{KS_Bochman2001}. 
Furthermore, we showed that $ \alpha $-equivalence is connected to the acceptance of contractionals.
To the best of our knowledge, contractionals have not been used in the context of iterated revision so far.

In future work we will explore the connection between Ramsey test conditionals and contractionals in the setting of iterated belief change;
for this, the recent work by Booth and Chandler on the correspondence between iterated revision and contraction  \cite{KS_BoothChandler2019}, employing closure operators over conditionals \cite{KS_ChandlerBooth2019}, will be useful.

\bibliographystyle{ecai}
\bibliography{bibexport}

\end{document}